\documentclass{article}

\usepackage{arxiv}
\usepackage{color} 
\usepackage[utf8]{inputenc} 
\usepackage[T1]{fontenc}    
\usepackage{hyperref}       
\usepackage{url}            
\usepackage{booktabs}       
\usepackage{amsfonts}       
\usepackage{nicefrac}       
\usepackage{microtype}      
\usepackage{lipsum}

\usepackage[export]{adjustbox}
\usepackage{geometry}
\usepackage{amsmath}
\usepackage{epsfig}
\usepackage{amssymb}
\usepackage[ruled,linesnumbered]{algorithm2e}
\usepackage{amsthm}

\usepackage{graphicx}
\usepackage{caption}
\usepackage{subcaption}
\graphicspath{ {./images/} }

\newtheorem{thm}{Theorem}[section]

\newtheorem{lem}[thm]{Lemma}
\def\x{{\mathbf x}}

\def\x{{\bf x}}

\def\y{{\bf y}}

\def\x{{\bf x}}
\def\y{{\bf y}}

\title{A Framework for Analyzing Cross-correlators using Price's Theorem and Piecewise-Linear Decomposition}

\author{
 Zhili Xiao \\
 Electrical \& Systems Engineering\\
 Washington University in Saint Louis\\
 Saint Louis, MO 63130 \\
  \texttt{xiaozhili@wustl.edu} \\
   \And
Shantanu Chakrabartty  \\
  Electrical \&  Systems Engineering\\
 Washington University in Saint Louis\\
  Saint Louis, MO 63130 \\
  \texttt{shantanu@wustl.edu} \\
}
\date{}
\begin{document}
\maketitle

\let\thefootnote\relax
\footnotetext{All correspondences related to this manuscript should be addressed to shantanu@wustl.edu.}

\begin{abstract}
Precise estimation of cross-correlation or similarity between two random variables lies at the heart of signal detection, hyperdimensional computing,  associative memories, and neural networks. Although a vast literature exists on different methods for estimating cross-correlations, the question {\it what is the best and the simplest method to estimate cross-correlations using finite samples ?} is still unclear. In this paper, we first argue that the standard empirical approach might not be optimal, even though the estimator exhibits uniform convergence to the true cross-correlation. Instead, we show that there exists a large class of simple non-linear functions that can be used to construct cross-correlators with a higher signal-to-noise ratio (SNR). To demonstrate this, we first present a general mathematical framework using Price's Theorem that allows us to analyze arbitrary cross-correlators constructed using a mixture of piece-wise linear functions. Using this framework and a high-dimensional mapping, we show that some of the most promising cross-correlators are based on Huber's loss functions, margin-propagation (MP) functions, and the log-sum-exponential (LSE) functions.
\end{abstract}

\section{Introduction}
Estimating cross-correlations between random variables play an important role in the field of statistics~\cite{statistics}, machine learning~\cite{LeCunCNN,Cross-correlationNN,ObjectRecognition,PatternCA}, and signal detection~\cite{FNCC,acoustic,radarscore, aeroradar}. This is because the cross-correlation metric measures some form of similarity between the random variables, revealing how one might influence the other. With proper normalization, the metric becomes equivalent to cosine similarity and unitary transforms, both of which are extensively used in linear algebra~\cite{linear}, natural language processing~\cite{NLP}, and computer vision~\cite{visuallearning,KernelCross-Correlator}. In computer vision and signal processing, cross-correlation is often used for feature extraction~\cite{Deepfake}, where higher precision implies more information is retained for further data processing and learning. Accurate and efficient cross-correlation for pattern recognition and template matching~\cite{ASC, LewisFNCC} also ensures reliable and real-time decision-making in applications like radar detection or object recognition. In the emerging field of hyperdimensional computing~\cite{KanervaHyperdimensional,TheoreticalHyperdim}, cross-correlations (or equivalently inner-products) are used for information retrieval from sparse distributed memories~\cite{Hassoun1993AssociativeNM, AttentionSDM}. Since most of the vectors in higher-dimensions are nearly orthogonal to each other, precision in cross-correlation is essential to discriminate between patterns and improve the robustness of learning~\cite{HPCFPGA,HPCprecision}.

In its most general form, cross-correlation $R:\mathbb{R} \times \mathbb{R} \rightarrow  \mathbb{R}$ is defined for a pair of random variables $X \in \mathbb{R}$, $Y \in  \mathbb{R} $ as 
\begin{equation}
    R \mathrel{:=} \mathcal{E}[XY] =  \int_{-\infty}^{\infty} \int_{-\infty}^{\infty} xy \enskip p(X=x,Y=y)dxdy,
    \label{eqn_defcorr}
\end{equation}
where $p:\mathbb{R} \times \mathbb{R} \rightarrow \mathbb{R}^+$ denotes the underlying joint probability distribution from which $x$ and $y$ are drawn from. The operator $\mathcal{E}[.]$ denotes an expectation under the probability measure $p$. 

In practice, the joint distribution $p$ is not known apriori. Instead, one has access to $N$ samples independently drawn from the distribution $p$, as illustrated in Fig.~\ref{fig:Motivation1}c. If we denote the sample vectors as $ \bf{x} \in \mathbb{R}^N$ and $\bf{y} \in \mathbb{R}^N$, then the cross-correlation is empirically estimated as 
\begin{equation}
\hat{R}_N =\frac{1}{N}\sum_{n=1}^{N} x_n y_n,
\label{eqn:empcorr}
\end{equation}
where $x_n \in \mathbb{R}$ and $y_n \in \mathbb{R}$ represent the elements of the vector $\bf{x}$ and $\bf{y}$. Then, by the law of large numbers (LLN), the empirical correlation converges uniformly to the true correlation $R$ 
\begin{equation}
\left|\frac{1}{N}\sum_{n=1}^{N} x_n y_n - R \right| \le \epsilon \stackrel{N \rightarrow \infty}{\longrightarrow} 0,
\label{eqn_empconverge}
\end{equation}
as depicted in Fig.~\ref{fig:Motivation1}. 

\begin{figure}[t]
  \centering
\includegraphics[width=\textwidth]{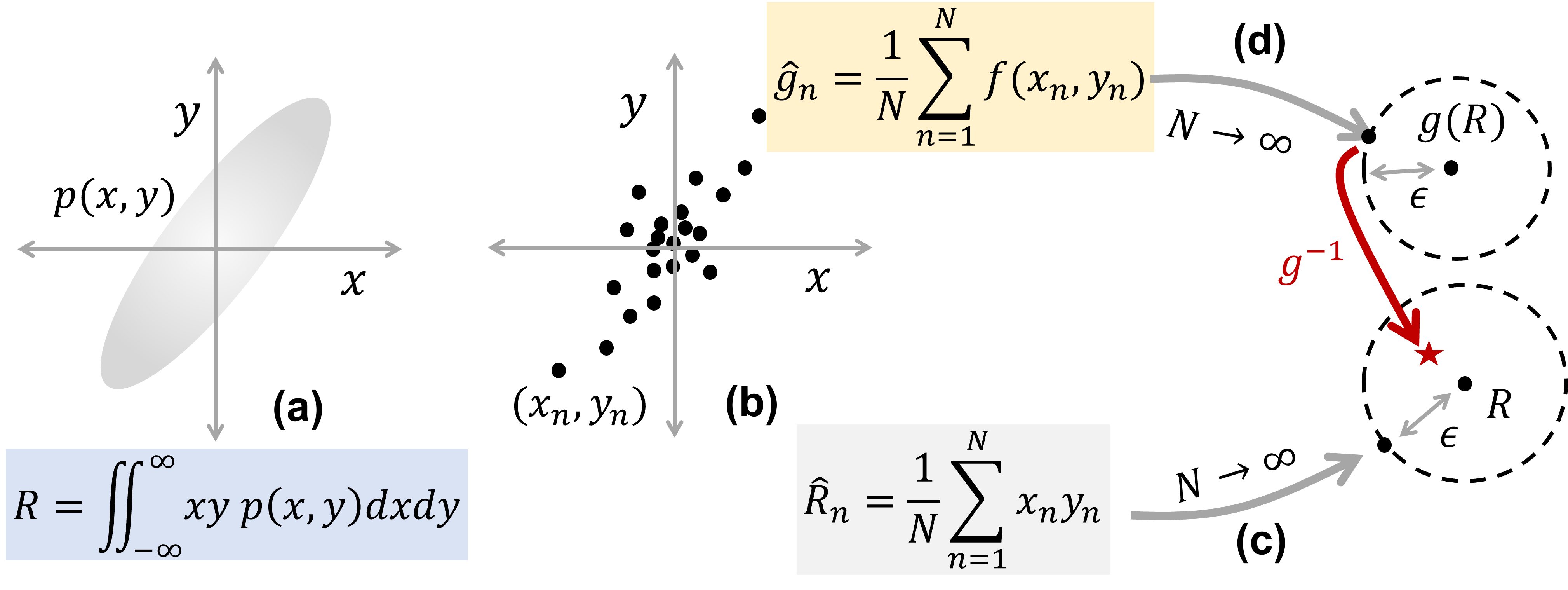}
  \caption{(a) Statistical definition of cross-correlation $R$ between random variables $x$ and $y$ under a joint distribution $p(x,y)$; (b) Empirical cross-correlation $\hat{R}_N$ based on samples $(x_n,y_n)$; (c) Uniform convergence of $\hat{R}_N$ to $R$; (d) Uniform convergence of empirical non-linear cross-correlator $\hat{g}_N$ to $g(R)$; and (e) Estimation of $R$ using $g^{-1}$.}
  \label{fig:Motivation1}
\end{figure}

In this paper, we refer to the above estimation approach using multiplication of samples as the empirical approach, and we explore an alternate approach towards estimating cross-correlations using a class of non-linear functions of samples $f:\mathbb{R} \times \mathbb{R} \rightarrow \mathbb{R}$ such that 
\begin{equation}
\frac{1}{N}\sum_{n=1}^{N} f(x_n,y_n) \stackrel{N \rightarrow \infty}{\longrightarrow} \mathcal{E}[f(X,Y)] = g(R),
\label{eqn_empnonlinear}
\end{equation}
where $g:\mathbb{R} \rightarrow \mathbb{R}$ is a monotonic function.

The uniform convergence 
of $f$ is illustrated in Fig.~\ref{fig:Motivation1} where 
\begin{equation}
\left|\frac{1}{N}\sum_{n=1}^{N} f(x_n,y_n) - g(R)\right| \le \epsilon \stackrel{N \rightarrow \infty}{\longrightarrow} 0.
\label{eqn_empnonlinear2}
\end{equation}
The main premise of this paper is that when $x$ and $y$ are drawn from a stationary distribution, the function $g$ is known apriori or can be estimated with high accuracy. As a result, for a finite sample size $N$, $g^{-1}(\frac{1}{N}\sum_{n=1}^{N} f(x_n,y_n))$ is an unbiased cross-correlation estimator, and its estimation could be closer than $\hat{R}$ to the true cross-correlation $R$, as illustrated in Fig.~\ref{fig:Motivation1}d. 

For the analysis and comparison presented in this paper, we will assume the following without loss of generality.
\begin{enumerate}
\item Both the random variables $X$ and $Y$ for which the cross-correlation is being estimated will be assumed to be zero mean $\mathcal{E}[X] = \mathcal{E}[Y] = 0$ and have unit variance $\mathcal{E}[X^2] = \mathcal{E}[Y^2] = 1$. In this case, the cross-correlation $R$ is equal to Pearson's correlation coefficient~\cite{pearson1895note}. Note that if $X,Y$ have nonzero means $\mu_x$ and $\mu_y$, the cross-correlation can be expressed as follow,
\begin{equation}
    \mathcal{E}[XY] = \mathcal{E}[\left(X - \mu_x\right)\left(Y-\mu_y\right)] + \mu_x \mu_y.
\end{equation}
Since $\mu_x \mu_y$ can be determined apriori, the accuracy of different cross-correlators is determined by the accuracy of the cross-correlation between the zero mean variables $X' = X - \mu_x$ and $Y' = Y - \mu_y$.

\item The non-linear function $f(x,y)$ used for generating different cross-correlators is passive and memory-less which implies that its output depends on the instantaneous values of $X$ and $Y$.  

\end{enumerate}

The paper is organized as follows: 
In section~\ref{sec:Analysis Framework}, we first propose the mathematical framework that can be used to analyze cross-correlators for a general class of non-linear function $f$ and for jointly Gaussian distributed inputs. We use the framework to analyze different types of estimators, which include the linear-rectifier cross-correlator, margin-propagation (MP) cross-correlators ~\cite{Chakrabartty2003MARGINPA}, Huber-type~\cite{Huber} cross-correlators, and log-sum-exponential (LSE)~\cite{LSEnetwork} estimators. In section~\ref{sec:extension to WHT}, we extend the framework to arbitrary input distributions based on the hyperdimensional mapping using the Walsh-Hadamard transformation. In section~\ref{sec:results}, we show experiments evaluating different correlators and the transformation method. Section~\ref{sec:Conclusions} discusses the advantages and disadvantages of different correlators and concludes the paper with a brief perspective on future directions.

\begin{figure*}[b]
    \centering
    \includegraphics[width=\textwidth]{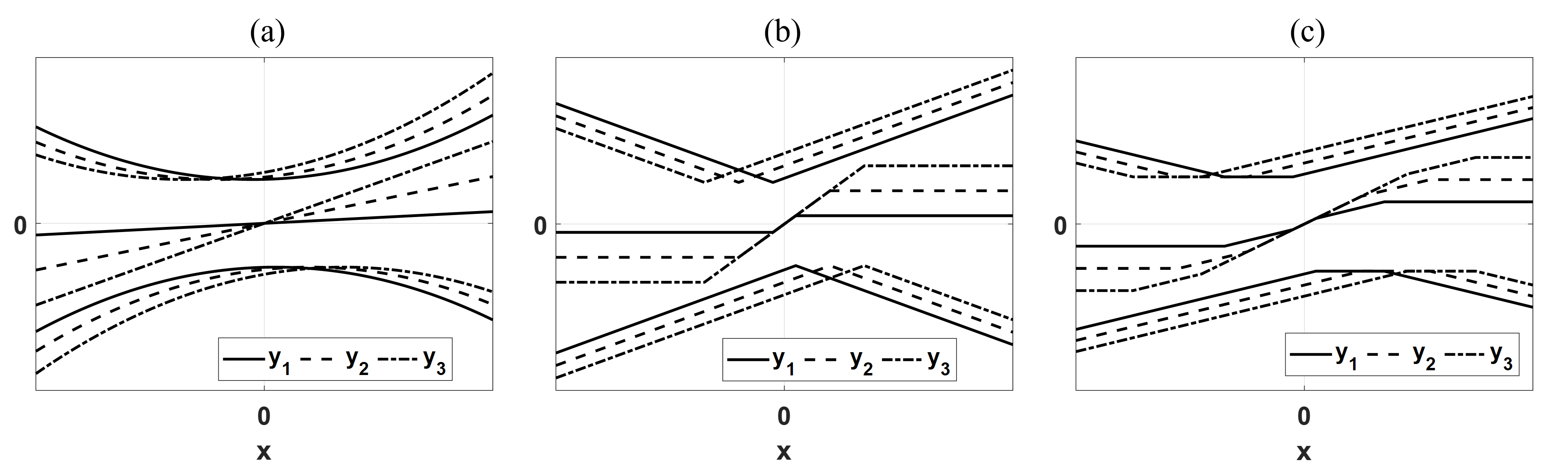}
    \caption{A conceptual scalar demonstration of a) empirical correlators, b) $L_1$ type correlators, and c) an example of correlators using mixtures of piece-wise linear functions with $L=1, w_1 = 1$, and $\alpha_1 = 0.5$. All correlators $f(x,y)$ can be viewed as a difference between $h(x+y)$ and $h(x-y)$. The upper three curves represent $h(x+y)$, the lower three curves represent $-h(x+y)$, and curves in the middle are $f(x,y)$ for different values of $y = y_1, y_2, y_3$. }
    \label{fig:scalar demonstration}
\end{figure*}


\section{Analysis Framework using Price's Theorem}
\label{sec:Analysis Framework}
In this section, we present an analytical framework that can be used to understand the behavior of different cross-correlators. A cross-correlator can be viewed as a difference between two functions, and in Fig.~\ref{fig:scalar demonstration}(a), we illustrate this for the empirical cross-correlator defined in equation~(\ref{eqn:empcorr}), which can be expressed as
\begin{equation}
\frac{1}{N}\sum_{n=1}^{N} x_n y_n = \frac{1}{4N}\sum_{n=1}^{N}(x_n + y_n)^2 - (x_n - y_n)^2.
\label{eqn:quadratic functions}
\end{equation}

The symmetric quadratic functions $(x+y)^2$ and $-(x-y)^2$ are shown in Fig.~\ref{fig:scalar demonstration}a, which results in the product $xy$ when summed together. When extended to $N$ dimensions, the quadratic functions in equation~(\ref{fig:scalar demonstration}) become $L_2$ distances $||\x+\y||_2^2$ and $||\x-\y||_2^2$ and their difference is proportional to the empirical cross-correlation $\frac{1}{N}\sum_{n=1}^N x_n y_n$. The concept can be generalized to other norms~\cite{AdderNet}. As an example, it has been suggested that substituting the quadratic component with linear rectifiers can yield a comparable cross-correlator~\cite{simplecorrelator,PriceExtension} which uses $L_1$ distances $|\x+\y|$ and $|\x-\y|$. Typically, the linear rectifier cross-correlator is considered a reasonable approximation of the empirical cross-correlator. In Fig.~\ref{fig:scalar demonstration}b, we show the equivalent construction for the $L_1$ type cross-correlators using $L_1$ distances $|\x+\y|$ and $|\x-\y|$. 

Both $L_2$ and $L_1$ constructions can be viewed as special cases of mixtures of piece-wise linear functions as shown in Fig.~\ref{fig:scalar demonstration}c and can be expressed as
\begin{equation}
    f(x,y) = h(x+y) - h(x-y),
    \label{eqn:f}
\end{equation}
where
\begin{equation}
	h(x) = \frac{1}{2}\sum_{l=1}^{L} w_l \left(\left|x - \alpha_l \right| + \left|x + \alpha_l \right| \right) ,
	\label{eqn:h}
\end{equation}
with parameters $\alpha_l \ge 0$, $w_l \ge 0; \sum_l w_l = 1$ and $\left|\cdot\right|$ is an absolute-value function defined as
\begin{equation}
	\left|x\right| = \left\{ \begin{array}{lcc}
		x   & ; &  x \ge 0 \\
		-x  & ; &  x < 0.
	\end{array} \right.
	\label{eq:absolute}
\end{equation}
We now state the Lemma that can be used to compute the function $g(R) = \mathcal{E}[f(x,y)]$.

\begin{lem}\label{lem1}
If the function $f(x,y)$ is given by equation~(\ref{eqn:f}) and~(\ref{eqn:h}) and the joint probability distribution of $X \in \mathbb{R}$ and $Y \in \mathbb{R}$ is given by
\begin{equation}
	p(X=x,Y=y; R) = \frac{1}{2\pi\sqrt{1-R^2}} \exp \left[-\frac{x^2 + y^2 - 2R x y}{2\left(1- R^2\right)} \right],
	\label{eqn:pdf}
\end{equation}
where $R \in [-1,1]$ is the Pearson's correlation/cross-correlation between $x$ and $y$,
then $g(R)=\mathcal{E}[f(X,Y)]$ is given by 
\begin{equation}
\begin{split}
	g(R) = \frac{1}{2\sqrt{\pi}}\sum_{l=1}^{L} \int_0^R  \frac{w_l}{\sqrt{(1+\rho)}}  \exp\left[-\frac{\alpha_l^2}{4\left(1+\rho\right)}\right]  + \frac{w_l}{\sqrt{(1-\rho)}} \exp\left[-\frac{\alpha_l^2}{4\left(1-\rho\right)}\right]  d\rho.
 \end{split}
    \label{eqn:lemma1 relation}
\end{equation}
\end{lem}
\begin{proof}
Since $f$ is a memory-less function with a well-defined Fourier transform and $X$ and $Y$ are zero-mean, unit variance, jointly distributed Gaussian random variables, we can apply Price's theorem~\cite{Price1,PriceExtension,CommentOnExtension} which states that
\begin{equation}
	\frac{\partial \mathcal{E}(f)}{\partial R} = \mathcal{E}\left[\frac{\partial^2 f}{\partial x \partial y}\right]
	\label{eq:price1},
\end{equation}
where the expectation operator $\mathcal{E}$ is defined as
\begin{equation}
	\mathcal{E}\left(f\right) = \int_{-\infty}^{\infty} \int_{-\infty}^{\infty} f(x,y)p(x,y; R)dx dy = g(R).
	\label{eq:pdf2}
\end{equation}
The partial derivatives of the sub-functions $h(x+y),h(x-y)$ in~(\ref{eqn:h}) are 
\begin{subequations}
\label{eqn:sub_sec_order_diff}
\begin{equation}
	\frac{\partial^2 h(x+y)}{\partial x \partial y} = \frac{1}{2}\sum_{l=1}^{L} w_l\left[\delta\left( x + y - \alpha_l \right) + \delta\left( x + y + \alpha_l \right)\right],
	\label{eqn:sec_order_diffp}
\end{equation}
\begin{equation}
	-\frac{\partial^2 h(x-y)}{\partial x \partial y} =  \frac{1}{2}\sum_{l=1}^{L} w_l\left[\delta\left( x - y - \alpha_l \right) + \delta\left( x - y + \alpha_l \right) \right],
	\label{eqn:sec_order_diffn}
\end{equation}
\end{subequations}
where $\delta(.)$ denotes the Dirac-delta function. The expectation operators can be computed as follow 
\begin{subequations}
\begin{align}
\mathcal{E} \left[ \frac{\partial^2  h(x+y)}{\partial x \partial y} \right] \nonumber 
&= \frac{1}{2}\sum_{l=1}^{L} w_l \int_{-\infty}^{\infty}  p(x,-x +\alpha_l) +  p(x ,-x -\alpha_l) dx
\label{eqn:Expected_sub_sec_order_diffp}
 \\
 &=\frac{1}{\sqrt{\pi(1+R)}} \sum_{l=1}^{L} w_l\exp \left[-\frac{\alpha_l^2}{4\left(1 + R\right)} \right]
 \\
-\mathcal{E} \left[ \frac{\partial^2  h(x-y)}{\partial x \partial y} \right] \nonumber 
&= \frac{1}{2}\sum_{l=1}^{L} w_l \int_{-\infty}^{\infty}  p(x,x -\alpha_l) +  p(x ,x +\alpha_l) dx
 \\
 &=\frac{1}{\sqrt{\pi(1-R)}} \sum_{l=1}^{L} w_l\exp \left[-\frac{\alpha_l^2}{4\left(1 - R\right)} \right].
\label{eqn:Expected_sub_sec_order_diffn}
 \end{align} \label{eqn:Expected_sub_sec_order_diff}
\end{subequations}
Substituting the results into~(\ref{eqn:f}) and~(\ref{eq:price1}) will get
\begin{equation}
\begin{split}
	\frac{\partial g}{\partial R} = \frac{1}{\sqrt{\pi(1+R)}}  \sum_{l=1}^{L} w_l\exp\left[-\frac{\alpha_l^2}{4\left(1+R\right)}\right] + \frac{1}{\sqrt{\pi(1-R)}} \sum_{l=1}^{L} w_l\exp\left[-\frac{\alpha_l^2}{4\left(1-R\right)}\right],
 \end{split}
 \label{eqn:dervprice}
\end{equation}
which leads to the expression for $g(R)$,
\begin{equation}
\begin{split}
    g(R) =  \int_{0}^{R} \frac{1}{\sqrt{\pi(1+\rho)}}  \sum_{l=1}^{L} w_l\exp\left[-\frac{\alpha_l^2}{4\left(1+\rho\right)}\right]d\rho
    + \int_{0}^{R}\frac{1}{\sqrt{\pi(1-\rho)}} \sum_{l=1}^{L} w_l\exp\left[-\frac{\alpha_l^2}{4\left(1-\rho\right)}\right] d\rho.
     \end{split}
    \label{eqn:finalprice}
\end{equation}
\end{proof}

{\bf Example 1:} When $L = 1, w_1 = 1, \alpha_1 = 0$, the function $f$ is reduced to 
\begin{equation}
    f(x,y) = |x + y| - |x - y|,
    \label{eqn:linear rectified correlator}
\end{equation}
which is the well-studied linear rectifier correlator. In this case, the relation~(\ref{eqn:lemma1 relation}) can be evaluated in closed form and is given by
\begin{equation}
    g_{L1}(R) = \frac{2}{\sqrt{\pi}}(\sqrt{1+R} - \sqrt{1-R}).
    \label{eqn:linear rectified output}
\end{equation}

{\bf Example 2:} When $w_l = 1/L, \alpha_l = c/L, l = 1,..,L$ and $c,L \rightarrow \infty$,  the function $f$ is reduced to 
\begin{equation}
    f(x,y) =  \frac{1}{2c}(\left(x + y\right)^2 -  \left(x - y\right)^2),
    \label{eqn:empirical correlator}
\end{equation}
where $c>|x|$ is the range of inputs. So the function $f$ becomes the empirical correlator. In this case, the summation in the relation~(\ref{eqn:lemma1 relation}) can be replaced by integrals in the limit $L \rightarrow \infty$ in which case 
\begin{eqnarray}
    \frac{\partial g_{L2}}{\partial R} = \frac{1}{c\sqrt{\pi(1+R)}} \int_0^{\infty} \exp\left[-\frac{x^2}{4\left(1+R\right)}\right] dx + \frac{1}{c\sqrt{\pi(1-R)}} \int_0^{\infty} \exp\left[-\frac{x^2}{4\left(1-R\right)}\right] dx = \frac{2}{c}. 
    \label{eqn:empiricalcorrelator}
\end{eqnarray}
Therefore, $g_{L2}(R) = \frac{2}{c}R$ matches the result for a scaled empirical cross-correlation. 

\begin{figure*}[t]
  \centering
         \includegraphics[width=0.85\textwidth]{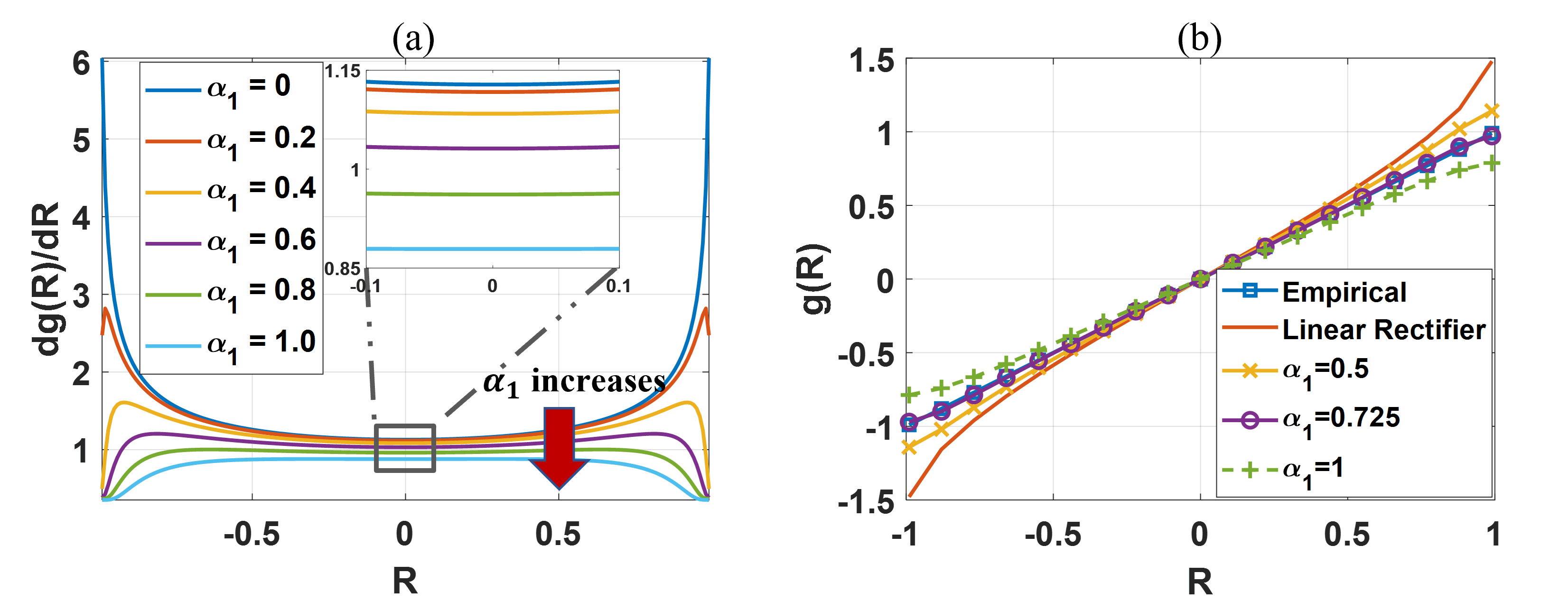}        \caption{a):$\frac{\partial g}{\partial R}$ for $L = 1, w_1=1$, and $\alpha_1 = \{0,0.2,0.4,0.6,0.8,1.0\}$; b) The $g(R)$ of the empirical and linear rectifier correlators and the $g(R)$ derived by integrating $\frac{\partial g}{\partial R}$ for $L = 1, w_1=1$, and $\alpha_1 = 0.5, 0.725$ and 1. It demonstrates how the analysis framework can be used to construct varying correlators with different output characteristics.}
     \label{fig:dev_offsets}
\end{figure*}

The above examples show how specific cross-correlators can be constructed using Lemma~\ref{lem1}. Even though the equation~(\ref{eqn:lemma1 relation}) may not be solved in closed-form, it is analytical and hence can be used to visualize the form of $g(R)$ for specific choices of $\alpha_l$ and $w_l$. Fig.~\ref{fig:dev_offsets}a visualizes the partial derivative $\frac{\partial g}{\partial R}$ in~(\ref{eqn:dervprice}) when $L = 1, w_1=1$, and $\alpha_1 = \{0,0.2,0.4,0.6,0.8,1\}$. The zoomed-in window demonstrates that the partial derivatives stay almost constant around zero correlation, which implies the linearity of $g(R)$ around zero correlation for any cross-correlators. Fig.~\ref{fig:dev_offsets}b displays the $g(R)$ of both empirical and linear rectifier correlators and the $g(R)$ for different values of a single offset $\alpha_1 = 0.5,0.725, 1$ computed by integration. The results will be verified by the Monte Carlo experiment for MP correlators in section~\ref{sec:results}.

Let's denote $\mathcal{E}\left(f\right)$ or $g(R)$ by $y$, we can derive the following from $g_{L1}(R)$ in~(\ref{eqn:linear rectified output})
\begin{equation}
    R^2 = \frac{\pi}{4} y^2 - \frac{\pi^2}{64}y^4. 
\end{equation}
This suggests that $g^{-1}$ in Fig.~\ref{fig:Motivation1} can be robustly estimated using a polynomial expansion with a relatively low degree. This is important since the closed-form solution for equation~(\ref{eqn:finalprice}) may not be computed for different choices of $w_l,\alpha_l$. As a result, $g^{-1}$ has to be learned/estimated by drawing samples with known apriori cross-correlation, which is then used to estimate $R$ according to Fig.~\ref{fig:Motivation1}. As we will show later in section~\ref{sec:extension to WHT}, this calibration procedure and procedure to estimate $R$ can be agnostic to the input distribution.

We now apply the calibration procedure to three other types of functions of the type given by expression~(\ref{eqn:f}). The first is the margin-propagation (MP) function 
given by 
\begin{align}
\label{eqn:MP}
    (x - z)_+ + (-x-z)_+ &= \gamma, \\
    h(x) &= z,
\end{align}
where $(\cdot)_+$ is the ReLU function, and $\gamma \ge 0$ is the hyper-parameter. It can be easily verified that the MP function in~(\ref{eqn:MP}) is equivalent to~(\ref{eqn:h}) with $L=1$ and $\alpha_1 = 0.5\gamma$.

The second function is the Huber function which requires an infinite number of splines and is given by
\begin{equation}
h(x) =
\begin{cases}
& 0.5x^2/\delta, \enskip x < \delta \\
& |x| - \frac{1}{2}\delta, \enskip x \geq \delta,
\end{cases}
\end{equation}
where $\delta > 0$ is a threshold parameter. 

The third function is a log-sum-exp (LSE) function which also requires an infinite number of splines and is given by
\begin{equation}
\label{eqn:LSE_h}
h(x) = \frac{1}{a}(\log(\exp\left[ax\right] + \exp\left[-ax\right]),
\end{equation}
where $a > 0 $ is a scaling factor. 

As we will show in section~\ref{sec:results}, the Huber and LSE correlators fall between the empirical and linear rectifier correlators, hence the inverse cross-correlation function $g^{-1}$ can be approximated by a polynomial of degree lower or equal to the degree needed for $g_{L1}^{-1}$. In practice, we found that fourth order polynomial is sufficient for calibration of $g_{L1}^{-1}$. For calibrating MP correlators, higher degree polynomials were needed as the value of $\gamma$ increases as its $|\frac{\partial g}{\partial R}|$ is not monotonically increasing as correlation increases. As such, a fifth-order polynomial was used to learn the inverse cross-correlation function $g^{-1}$.

\section{Extension to non-Gaussian distributions}
\label{sec:extension to WHT}
The theoretical results presented in section~\ref{sec:Analysis Framework} assumed random variables with joint Gaussian distributions. In this section, we extend the previous results for non-Gaussian distributions. To achieve this, we use results from the hyperdimensional computing literature, which state that variances and cross-correlations are preserved when random variables are mapped into high-dimensional space using unitary random matrices. 

\begin{lem}\label{lem3}
Let $X$ and $Y$ be zero-mean random variables with unit variance and with a cross-correlation $R$. Let $\Phi: \mathbb{R}^N \rightarrow \mathbb{R}^M$ denote a high-dimensional embedding using a Unitary transform such that $\mathcal{E}[\Phi(\x)] = \mathbf{0}$. Then, as $N \rightarrow \infty$, $\frac{1}{N}\left<\Phi(\x), \Phi(\y)\right> \rightarrow  R$, where $\left< \cdot, \cdot\right>$ is the inner product.
\end{lem}
\begin{proof}
     Suppose $\x$ and $\y$ are $\mathbb{R}^N$-valued random vector, and each entry $x_n$, $y_n$ are independently identically distributed (i.i.d) variables with the joint probability density function $p(X=x_n,Y=y_n; R)$
    \begin{align}
    \label{eqn:transformed_R}
    \frac{1}{N}\left<\Phi(\x), \Phi(\y)\right> =  \frac{1}{N}\left<\x, \y\right> = \frac{1}{N}\sum_{n=1}^{N} x_n y_n \stackrel{N \rightarrow \infty}{\longrightarrow} R.
    \end{align}
\end{proof}

The Walsh-Hadamard-Transform is one such unitary transform $\mathcal{H}: \mathbb{R}^{N} \rightarrow \mathbb{R}^{N}$ and it can be represented by a $N \times N$  Hadamard matrix. The transformation generally requires the input zero-padded to a power of two. An example $4 \times 4$ WHT matrix is shown below
\begin{equation}
H_4 = \frac{1}{2}\begin{pmatrix}
1 & 1 & 1 & 1 \\
1 & -1 & 1 & -1 \\
1 & 1 & -1 & -1 \\
1 & -1 & -1 & 1
\end{pmatrix}.
\label{eqn:hadamard matrix}
\end{equation}
The Hadamard matrices are orthogonal and symmetric matrices composed of +1 and -1 with a normalization factor $1/\sqrt{N}$, which makes it easy for implementations and computations. Besides keeping the covariance between random variables, it can also be shown that the transformed zero-mean variables converge to the joint Gaussian distribution with the same variance and covariance.

\begin{lem}\label{lem4}
Let $X$ and $Y$ be zero-mean random variables with finite variances and the cross-correlation $R$. Let $\mathcal{H}: \mathbb{R}^{N} \rightarrow \mathbb{R}^{N}$ denote the Walsh-Hadamard-Transform. Suppose the entries of the vector equation $\bf{x'} = \mathcal{H}(\bf{x})$ are given by $x_n' = h_n(x_1, ..., x_n)$, and $h_n(\cdot)$ is therefore
\begin{equation}
        h_n(x_1, ..., x_n) =
\begin{cases}
& \frac{1}{\sqrt{N}}\sum_{n=1}^{N}x_n,  \enskip \enskip \enskip n = 1, \\
& \frac{1}{\sqrt{N}}(\sum_{n=1}^{N/2}x_n - \sum_{n=1}^{N/2}x_n), \enskip n \neq 1,
\end{cases}
\end{equation}

where the $x_n$ are independently identically distributed (i.i.d.) samples from $X$. Then, as $N \rightarrow \infty$, the joint probability distribution $p(x_n', y_n')$ converges to a bivariate Gaussian distribution with zero-mean, same variances $\sigma^2_x$, $\sigma^2_y$, and covariance $R$.
\end{lem}

\begin{proof}
     Suppose $X$ and $Y$  are zero-mean random variables with finite variances and covariance $R$. According to the multivariate Central Limit Theorem (CLT)~\cite{Klenke}, as $N \rightarrow \infty$
    the joint distribution $p(\sqrt{N}\bar{x}_N, \sqrt{N}\bar{y}_N)$ 
    converges to bivariate Gaussian distribution with zero mean and the same variances and covariance R, where $\bar{x}_N = \frac{1}{N}\sum_{n=1}^{N}x_n$ is the average of $N$ independently identically distributed samples of $X$. 
    
    Notice that the transformed entry after WHT $x_n' = h_n(x_1, ..., x_n)$ can be expressed as 
\begin{equation}
        x_n' =
\begin{cases}
& \sqrt{N}\bar{x}_N,  \enskip \enskip \enskip n = 1, \\
& \frac{1}{\sqrt{2}}\sqrt{\frac{N}{2}}(\bar{x}_{N/2} - \bar{y}_{N/2}), \enskip n \neq 1.
\end{cases}
\end{equation}
For $n=1$, the CLT can be applied to the transformed input $x_1'$ and $y_1'$. For the case of $n \neq 1$, note that the CLT also applies to $\pm \sqrt{\frac{N}{2}}\bar{x}_{N/2}$ and $\pm \sqrt{\frac{N}{2}}\bar{y}_{N/2}$, so they become bivariate Gaussian with zero means, and same variances and covariance, and so is their sum divided by $\sqrt{2}$.
\end{proof}
As such, using the proof in section~\ref{sec:Analysis Framework}, it can be shown that for correlators that can be expressed by equations ~(\ref{eqn:f}) and~(\ref{eqn:h}), the expected output $\mathcal{E}[f(\mathcal{H}(\x), \mathcal{H}(\y))]$ is equal to $g(R)$. In other words, for non-Gaussian distributed variables with zeros mean and finite covariance $R$, we can first transform the inputs to jointly Gaussian distribution, which preserves the covariance, and then use the cross-correlator in section~\ref{sec:Analysis Framework} to estimate the cross-correlation $R$ using the transformed data and the same $g^{-1}$.

\section{Experiments Results and Analysis}\label{sec:results}
In the first part of the experiments, we validate the analytical results in section~\ref{sec:Analysis Framework} when the random variables are drawn from joint Gaussian distribution. For this specific case, a closed-form expression of the Cramér–Rao bound can be computed. This bound can then be used to evaluate the effectiveness of any unbiased cross-correlation estimators. We will then validate our analytical results for arbitrary input distributions.

\subsection{Cross-correlation Cramér–Rao bound}
\label{sec:min var estimator}
According to the Cramér–Rao bound (CRB)~\cite{CRbound}, the variance or standard deviation of any unbiased estimator (satisfying regularity conditions) is lower bounded by the reciprocal of the Fisher information~\cite{FisherOnTM}. The following lemma presents the Cramér–Rao bound of the correlation estimator under the standard bivariate normal distribution.

\begin{lem}\label{CRB_lem}
For a collection of $N$ independent and identically distributed (iid) bivariate variables $X \in \mathbb{R}$ and $Y \in \mathbb{R}$ drawn from the jointly Gaussian distribution defined by~(\ref{eqn:pdf})  the Cramér–Rao bound of the correlation estimator is given by
\begin{equation}
     \sigma^2(\hat{R}) \ge \frac{(1-R^2)^2}{N(1+R^2)}.
\end{equation}
\end{lem}
\begin{proof}
For the distribution~(\ref{eqn:pdf}), the Fisher information for the correlation coefficient of a single pair of samples $(x,y)$ can be computed as,
\begin{equation}
I(R) = -\mathcal{E}\left[ \frac{\partial^2(l(x,y;R))}{\partial R^2}\right],
\end{equation}
where $l(x,y;R)$ represents the natural logarithm likelihood function of a single sample pair $(x,y)$ for the joint density function $p(X=x,Y=y;R)$ 
\begin{equation}
    l(x,y;R) = -\log(2\pi)-\log(\sqrt{1-R^2})-\frac{x^2 + y^2 - 2R x y}{2\left(1- R^2\right)}.
    \label{eqn:loglikelihood}
\end{equation}
From~(\ref{eqn:loglikelihood}) and using
$\mathcal{E}[X^2]=\mathcal{E}[Y^2] = 1$ and $\mathcal{E}[XY] = R$ leads to the Fisher's information metric
\begin{equation}
    I(R) = \frac{1 + R^2}{(1-R^2)^2}.
\end{equation}
For $N$ independent and identically distributed (iid) samples, the total Fisher information is the sum of information from each individual sample. Therefore, the Cramér–Rao bound is given by 
\begin{equation}
    \sigma^2(\hat{R}) \ge \frac{1}{I(R)} = \frac{(1-R^2)^2}{N(1+R^2)}.
\end{equation}
\end{proof}
In other words, the accuracy of all cross-correlators is upper bounded by the Fisher information. In section~\ref{sec:results}, we use the standard deviation of estimation errors $\sigma$ to assess the performance of cross-correlators.

\subsection{Monte-Carlo Experiments for Jointly Gaussian Inputs}\label{sec:true corr results}
This section presents the results using different correlators to estimate the covariance for standard jointly normal distribution. To study and compare their performance, vectors of different lengths are randomly sampled from the zero mean and unit variance Gaussian distribution. Each vector pair $\bf{s_1}, \bf{s_2} \in \mathbb{R}^N$ is mixed in the following way to generate a bivariate Gaussian distribution $\bf{x} = (\bf{x_1},\bf{x_2})$ with different correlation $R$ to learn and test the inverse cross-correlation function $g(R)$,
\begin{align}
    \bf{x_1} &= \bf{s_1} ,\\
    \bf{x_2} &= R \bf{s_1} + \sqrt{1-R^2}  \bf{s_2},\\
    \bf{x} &= (\bf{x_1}, \bf{x_2}) \sim N(\vec{0}, \Sigma), \Sigma = \begin{pmatrix}
1 & R \\
R  & 1
\end{pmatrix}.
\end{align}

\begin{figure}[b]
  \centering
         \includegraphics[width =\textwidth]{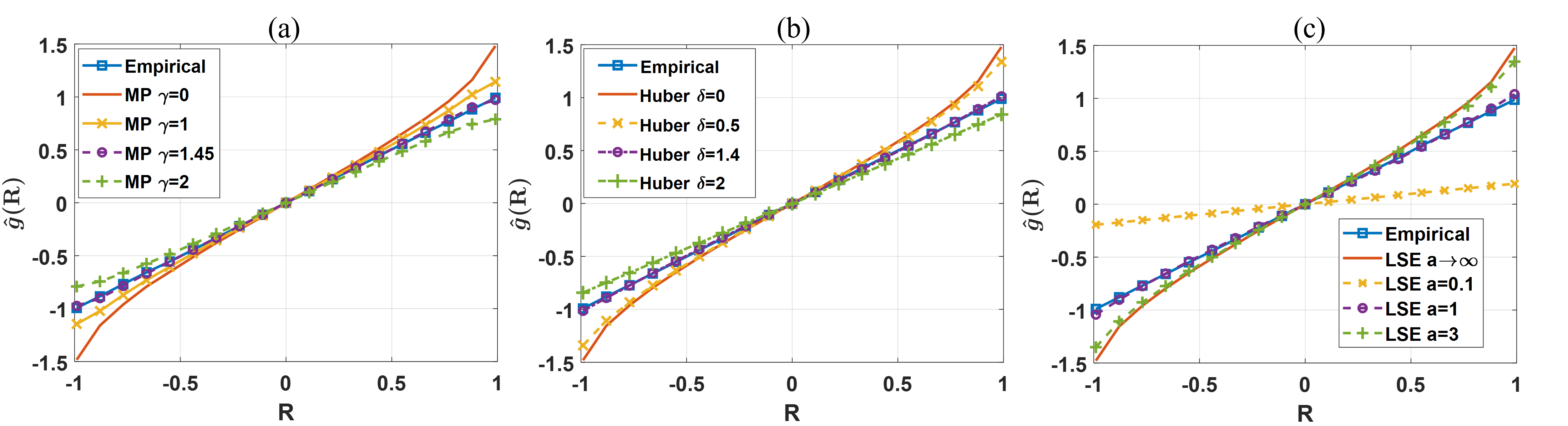}
         \caption{The estimated $\hat{g}(R)$ for the MP, Huber, and LSE functions for the standard bivariate normal distribution from Monte Carlo experiments. The expected outputs of the empirical and linear rectifier correlators are added for comparison. (a) The $\hat{g}(R)$ of MP correlators with $\gamma = \{0,1,1.45,2\}$; (b) The $\hat{g}(R)$ of Huber correlators with $\delta = \{0.5,1.4,2\}$; (c) The $\hat{g}(R)$ of LSE correlators with $a = \{0.1, 1, 3\}$.}
     \label{fig:gradient validation}
\end{figure}

Fig.~\ref{fig:gradient validation} shows the average output of cross-correlation function $\hat{g}(R)$ corresponding to the MP, Huber, and LSE functions with different parameters, which is used as an approximation of the correlation function $g(R)$ to get the calibration function $g^{-1}$. The output of the empirical correlator is shown for comparison. Note that the linear rectifier correlator is a special case for the MP, Huber, and LSE functions, which is included and labeled as "MP $\gamma=0$," "Huber $\delta=0$," and "LSE $a\rightarrow\infty$."
In fact, the normalized $g(R)$ for the Huber function and LSE functions are bounded above and below by the normalized  $g_{L1}(R)$ and $g_{L2}(R)$. This is easy to see for Huber functions, as it's a combination of the quadratic function ($L_2$) and absolute value function ($L_1$). For LSE functions, as the scaling factor $a$ increases, the $h(x)$ for LSE functions in expression~(\ref{eqn:LSE_h}) can be simplified to $\frac{1}{a}log(exp\left[a|x|\right] = |x|$ as the negative part will go to zero exponentially fast. On the other hand, as $a$ decreases, we can Taylor expand the exponential function at zero, which leads to
\begin{align}
h(x) = \frac{1}{a}\log(\sum_{n=0}^{\infty}\frac{(ax)^n}{n!} + \frac{(-ax)^n}{n!} ). 
\end{align}
The odd-degree terms cancel each other out and higher-order terms decay fast, which leads to 
\begin{align}
h(x) \approx \frac{1}{a}\log(2 + (ax)^2 ).
\end{align}
Applying the same trick to the logarithm but expanding at 2, we have
\begin{align}
    h(x) \approx \frac{1}{a} (\log(2) + \frac{1}{2}(ax)^2 + ...) \approx \frac{1}{a} + \frac{1}{2}ax^2.
\end{align}
Therefore, the LSE correlator approaches the empirical correlator as $a$ decreases.

\begin{figure*}[t]
  \centering
         \includegraphics[width=\textwidth]{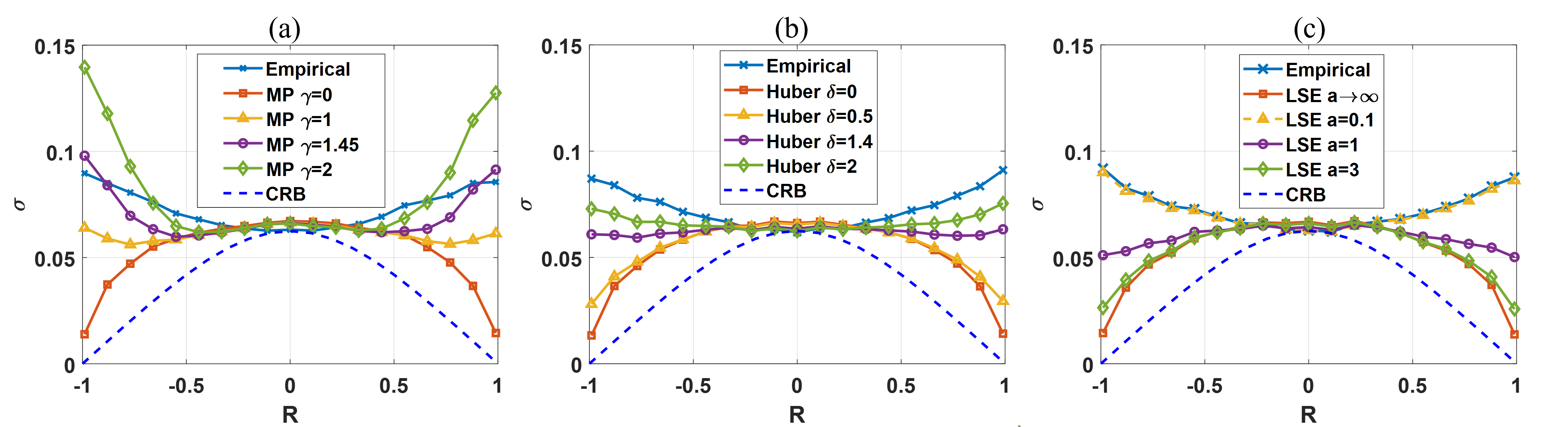}
         \caption{The standard deviation of the cross-correlation estimation error of MP, Huber, and LSE correlators for different correlation levels and the CRB at the dimension of 256, with the SNR of empirical and linear rectifier added for reference. (a) The error plot of MP correlators with $\gamma = \{1,1.45,2\}$; (b) The error plot of Huber correlators with $\sigma = \{0.5,1.4,2\}$; (c) The error plot of LSE correlators with $a = \{0.1, 1, 3\}$.}
     \label{fig:residue plots}
\end{figure*}

\begin{figure*}[b]
  \centering
\includegraphics[width=\textwidth]{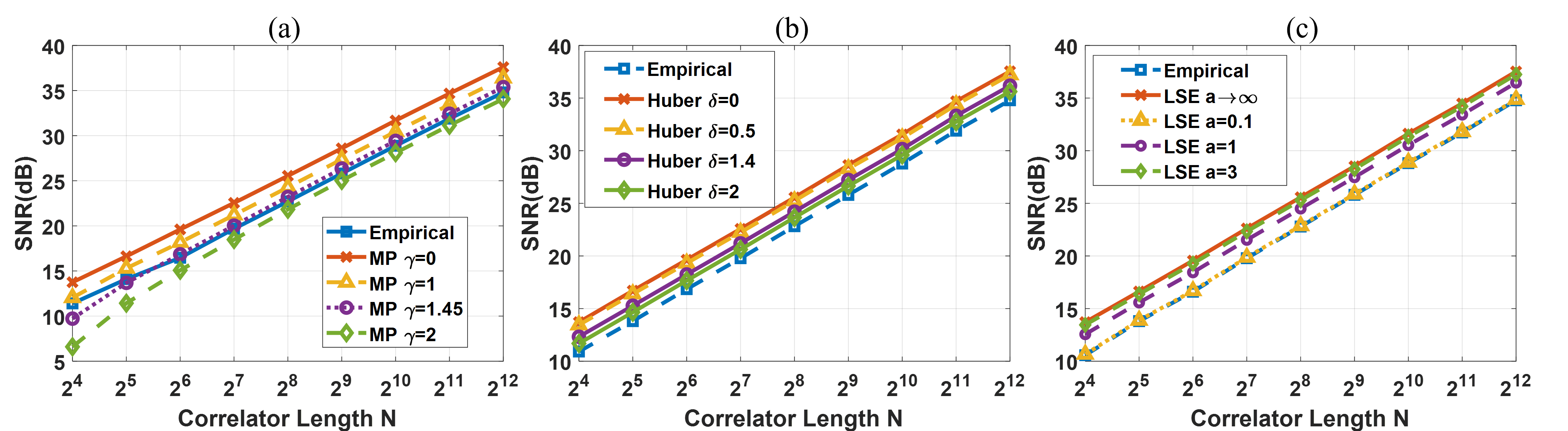}
         \caption{The SNR plots of MP, Huber, and LSE correlators from correlator length of 16 to 65536, with the SNR of empirical and linear rectifier correlators as a reference. a) The SNR plot of MP correlators with $\gamma = \{1,1.45,2\}$; b) The SNR plot of Huber correlators with $\sigma = \{0.5,1.4,2\}$; c) The SNR plot of LSE correlators with $a = \{0.1, 1, 3\}$. }
     \label{fig:SNR plot}
\end{figure*}
As discussed in section~\ref{sec:min var estimator}, the standard deviation of estimation can be used to evaluate the performance of correlators. Since the expected value of $g^{-1}$ is equal to $R$, the standard deviation of estimation $\sigma(\Bar{R})$ is equivalent to the standard deviation of estimation error $\sigma(R - g^{-1}(f))$. In Fig.~\ref{fig:residue plots}, we display the standard deviation of cross-correlation estimation errors and its CRB lower bound, denoted by $\sigma$, made by the MP, Huber, and LSE correlators with different parameters at different levels of cross-correlation using the learned inverse cross-correlation function $g^{-1}$. It is observed that the linear rectifier correlator is more accurate when the signal of interest is highly correlated, while the empirical correlator makes less error in the other case. The error profile of the Huber and LSE correlators becomes more similar to the empirical correlator when $\sigma$ is high and $a$ is small and approaches the linear rectifier correlator otherwise. The MP correlator in~(\ref{eqn:MP}) is equivalent to the linear rectifier correlator when $\gamma=0$. As $\gamma$ increases, the performance degrades and performs even worse than the empirical correlator.

In Fig.~\ref{fig:SNR plot}, we plot the signal-to-noise-ratio (SNR) for these cross-correlators for different correlator length $N$, where the SNR is defined as:
\begin{equation}
    \text{SNR} =  20\log_{10}(\frac{1}{\Bar{\sigma}}),    
\end{equation}
where $\Bar{\sigma}$ is the average standard deviation of the error across multiple Monte-Carlo runs. The result shows that the linear rectifier correlator has the best SNR in joint Gaussian input distribution. Also, the SNR increases by 3dB when the correlation length $N$ is doubled, which can be attributed to the reduction in the estimation error due to simple averaging.

\begin{figure*}[b]
  \centering
\includegraphics[width=\textwidth]{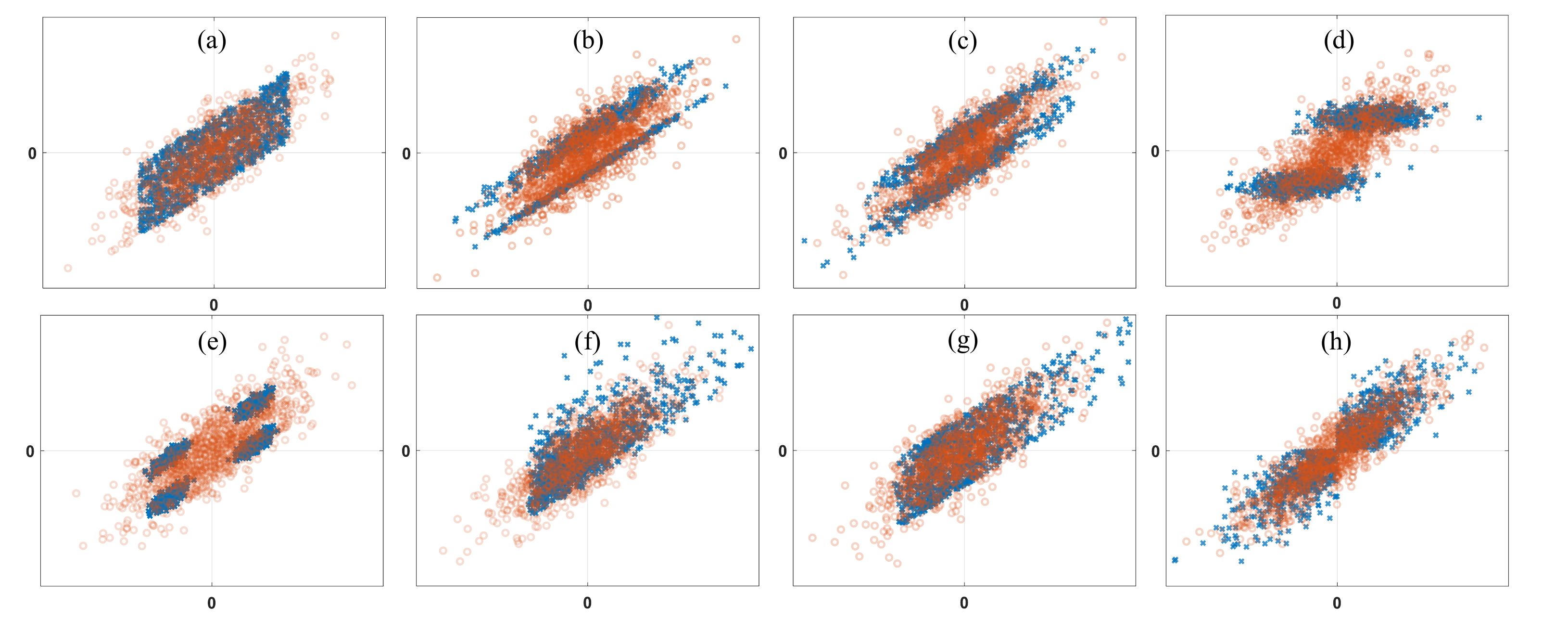}
         \caption{Scatter plots showing the random variables $x$ and $y$ drawn from different probability distributions with zero-mean, unit variance and a correlation R=0.8. The data points denoted by the blue cross are drawn from the original distribution and the data points denoted by the red circles are the inputs after applying WHT. After applying WHT, the resulting distribution converges towards a standard bivariate normal distribution with a correlation $R$=0.8.
         }
     \label{fig:distribution_plots}
\end{figure*}

\subsection{Monte-Carlo Experiments with WHT and Non-Gaussian Inputs}
\label{sec:Non-Gaussian exp}
In this section, the WHT method in section~\ref{sec:extension to WHT} was tested on non-Gaussian distributions to verify that the function $g(R)$ remains unchanged for non-Gaussian inputs after being transformed by the WHT. 
Fig.~\ref{fig:distribution_plots} shows varying joint input distributions before and after the WHT, which are denoted by blue crosses and red circles, respectively. The non-Gaussian distributions before the transformation are combinations of the uniform, Gaussian, and gamma distributions, and the inputs have a correlation of 0.8. It can be seen that the vectors become jointly Gaussian distributed after the transformation, and their correlation is retained. Monte-Carlo experiments show that the expected correlator output $\mathcal{E}[f(\mathcal{H}(x), \mathcal{H}(y))]$ is the same as the $g(R)$ for jointly Gaussian inputs. So no calibration is needed to learn $g^{-1}$ for different distributions.

On the other hand, it should be noticed that the standard deviation for different cross-correlation estimators is not guaranteed to be the same for different input distributions, even after the WHT transformation. To see this, notice that the WHT process will not change the output for the empirical correlator because the WHT transformation is unitary. However, the standard deviation of the empirical correlator, which is given by
\begin{equation}
\begin{split}
\mathcal{E}(XY - \mathcal{E}[XY])^2 = \mathcal{E}[X^2Y^2] - g^2(R)
= \int_{-\infty}^{\infty} \int_{-\infty}^{\infty} x^2y^2  \enskip p(x,y)dxdy - g^2(R),
\end{split}
\end{equation}
will change as the joint probability density function $p(x,y)$ changes.  In particular, for joint Gaussian distribution, it can be verified that $\mathcal{E}(xy - \mathcal{E}[xy])^2 = 1 + R^2$, whereas it is $1-\frac{1}{5}R^2$ in the case of jointly uniform distribution. As an example, Fig.~\ref{fig:sym_residue plots} shows the standard deviation of estimation error plots for the joint uniform distribution as shown in Fig.~\ref{fig:distribution_plots}a. It is observed that the contour of the error plot changes for all correlators except for the linear rectifier correlator, which is still the best-performing cross-correlation estimator in this specific non-Gaussian input distribution.
\begin{figure*}[t]
  \centering
         \includegraphics[width=\textwidth]{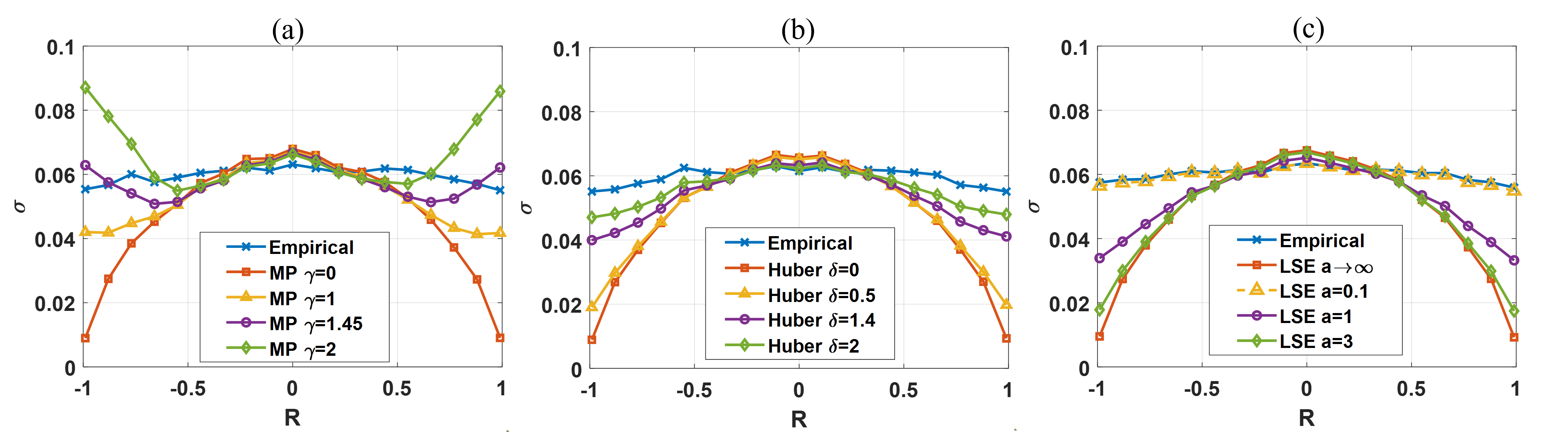}
    \caption{The standard deviation error plot of MP, Huber, and LSE correlators for the symmetric input distribution test when dimension is 256. a) The error plot of MP correlators with $\gamma = \{1,1.45,2\}$; b) The error plot of Huber correlators with $\sigma = 0.5,1.4,2\}$; c) The error plot of LSE correlators with $a = \{0.1, 1, 3\}$.}
     \label{fig:sym_residue plots}
\end{figure*}

\section{Discussions and Conclusions}
\label{sec:Conclusions}
In this paper, we present a mathematical framework for analyzing different types of cross-correlators. The closed-form solution facilitates the comparison of different cross-correlators and allows us to understand performance trade-offs. The analysis framework has been verified by Monte Carlo simulation for different input distributions. The error analysis reveals that the shape of the error profile exhibits a trade-off. Cross-correlation estimators that exhibit high errors near $R \approx 0$ make fewer errors near $|R| \approx 1$. The hyperparameters of the Huber estimators, MP estimators, and LSE estimators can be adapted to achieve different error profiles.

However, the complexity of implementing these different online estimators on hardware could be significantly different. The empirical and Huber cross-correlator relies on the quadratic function and hence may be difficult to implement on hardware. On the other hand, the MP and LSE correlators can be easily implemented on analog hardware~\cite{MingGuCMOS}. Another benefit of employing other correlators is the potential for improved computational efficiency and wider dynamic range when implementing them on digital systems. It is evident that the linear rectifier correlator and the MP correlator (without additional offsets) are more cost-effective than the empirical correlator since each multiplication is replaced by three and five addition operations~\cite{AbhishekMPFPGA}. The computation of the linear rectifier can be further simplified to condition and shift operations, and the addition operation is resistant to underflow on a fixed-point system. Moreover, the LSE correlator exhibits superior numerical stability at the expense of computational complexity, a common technique employed in machine learning to address the issue of gradient updates.

If an accurate estimate of the covariance or Pearson's correlation is necessary, then the calibration is required to learn the inverse function  $g^{-1}$ for inputs with diverse distributions and variances. To this end, the WHT is proposed as a pre-processing technique to convert input distributions to joint normal distributions, allowing the calibration process to be indifferent to input distributions. The limitation is that the transformation requires the dimension of inputs to be a power of two. However, calibration is still needed for varying variances, which inevitably introduces extra computational and resource expenses. On the other hand, the raw correlator output $f(\cdot)$ would be sufficient if the objective of the task is to obtain only a similarity score for different signals. For instance, in hyperdimensional computing, whether some $\bf{y} \in \mathbb{R}^N$ is contained in a set $\mathcal{S} \in \mathbb{R}^N$ is checked by if the dot product $\left<\phi(y),\phi(\mathcal{S})\right>$ is above a certain threshold, where $\phi(\cdot)$ is the hyperdimensional representation. The WHT can be incorporated into the hyperdimensional mapping process $\phi$, and alternative correlation functions, such as those discussed, can be used instead of resource-intensive and computationally inefficient inner products.

Regarding the accuracy for estimating the cross-correlation of jointly Gaussian distributed inputs, it appears that the empirical method may not be the most effective correlator. The linear rectifier correlator is superior in estimating the covariance for highly correlated signals and in terms of overall SNR. Its standard deviation of error plot has a similar trend with the Cramér–Rao bound (CRB). The performance gained for high correlation can be explained by the shape of $g(R)$. The variance of output $f(x,y)$ is relatively small with respect to the gradient of $g(R)$ for high correlation, which results in a larger confidence interval for estimations. The Huber and LSE correlator's performance is bounded by the linear rectifier and empirical correlator. The MP correlator in~(\ref{eqn:MP}) is equivalent to the linear rectifier correlator when $\gamma=0$. As $\gamma$ increases, its performance for higher $\gamma$ values can be potentially improved by introducing offsets. Of course, the above observations are not guaranteed to hold for other input distributions and are left for future research. 

\bibliographystyle{unsrt}  
\bibliography{references}

\end{document}